\documentclass{article}
\usepackage{fullpage}

\usepackage{amsmath,amsthm,mathtools}
\usepackage{graphicx}
\usepackage{algorithm2e}
\usepackage{thm-restate}

\usepackage{tikz}
\usetikzlibrary{shapes, calc, arrows, through, intersections, decorations.pathreplacing, patterns}

\newtheorem{theorem}{Theorem}
\newtheorem{definition}[theorem]{Definition}

\newtheorem{lem}[theorem]{Lemma}
\newtheorem{corollary}[theorem]{Corollary}

\newcommand{\mb}{\mathbf}
\newcommand{\mc}{\mathcal}
\newcommand\numberthis{\addtocounter{equation}{1}\tag{\theequation}}
\let\oldnl\nl
\newcommand{\nonl}{\renewcommand{\nl}{\let\nl\oldnl}}

\DeclareMathOperator*{\argmin}{arg\,min}
\DeclareMathOperator*{\vcdim}{VC-Dim}

\title{Semi-supervised clustering for de-duplication}

\author{\normalsize
{Shrinu Kushagra} {\textnormal {,}} {Shai Ben-David} {\textnormal {and}} {Ihab Ilyas} \\
\normalsize David R. Cheriton School of Computer Science \\
\normalsize University of Waterloo, Waterloo, Ontario, Canada\\
\normalsize \{skushagr,shai,ilyas\}@uwaterloo.ca \\
}
\date{}
\begin{document}
\maketitle
\begin{abstract}
Data de-duplication is the task of detecting multiple records that correspond to the same real-world entity in a database. In this work, we view de-duplication as a clustering problem where the goal is to put records corresponding to the same physical entity in the same cluster and putting records corresponding to different physical entities into different clusters.

We introduce a framework which we call promise correlation clustering. Given a complete graph $G$ with the edges labelled $0$ and $1$, the goal is to find a clustering that minimizes the number of $0$ edges within a cluster plus the number of $1$ edges across different clusters (or correlation loss). The optimal clustering can also be viewed as a complete graph $G^*$ with edges corresponding to points in the same cluster being labelled $0$ and other edges being labelled $1$. Under the promise that the edge difference between $G$ and $G^*$ is ``small", we prove that finding the optimal clustering (or $G^*$) is still NP-Hard. \cite{ashtiani2016clustering} introduced the framework of semi-supervised clustering, where the learning algorithm has access to an oracle, which answers whether two points belong to the same or different clusters. We further prove that even with access to a same-cluster oracle, the promise version is NP-Hard as long as the number queries to the oracle is not too large ($o(n)$ where $n$ is the number of vertices). 

Given these negative results, we consider a restricted version of correlation clustering. As before, the goal is to find a clustering that minimizes the correlation loss. However, we restrict ourselves to a given class $\mc F$ of clusterings. We offer a semi-supervised algorithmic approach to solve the restricted variant with success guarantees. 
\end{abstract}

\section{Introduction}
Record de-duplication is a central task in data cleaning in large data bases. Common practical examples include the detection of records referring to the same patient in large health data bases (different records might have been generated for same patient in different clinics or even in the same clinic at different times), detecting same person records in census data, detecting customer records, duplicate records of papers in Google Scholar and so on and so forth \cite{elmagarmid2007duplicate}, \cite{chu2016distributed}, \cite{ilyas2015trends}.

Since the same-entity relation is reflexive symmetric and transitive, the sets of duplicate records can be viewed as clusters. Consequently, the record de-duplication task can be viewed as a clustering task. Such a clustering task has several characteristics that make it hard to address with common clustering tools; The number of ground truth clusters is unknown to the algorithm. Furthermore, one cannot a priory bias the algorithm towards a larger or a smaller number of clusters (unlike, say, facility location tasks in which it makes sense to trade off cluster cohesiveness with the number of clusters). This implies that attempts to use standard classification prediction learning tools to predict which pairs or records should be labelled `same-cluster' and which should be `different clusters' (D) are bound to fail - uniformly drawn samples of pairs are likely to be all labeled D and the resulting constant "all D" classifier will have negligible $0-1$ error over the set of pairs. On top of that all, there is no a priori geometry to the clusters structure, one cannot justify common simplifying assumptions like some stability of the clustering, some convexity of the larger-than-two sized clusters  or significant between-clusters margins.

The framework of correlation clustering extends naturally to the data de-duplication problem. Given a complete graph $G$ vertices where each edge is labelled as a $1$ or a $0$, the goal is to cluster the vertices of the graph so as to correlate “as much as possible” to the edges of the graph. That is find a clustering so as to minimize the number of $0$ edges within a cluster plus the number of $1$ edges across different clusters (or correlation loss). An edge label $0$ indicates that the two records have been deemed to be different while $1$ indicates that the records are similar. However, finding the clustering with minimum correlation loss is known to be NP-Hard \cite{bansal2004correlation}. 

One characteristic of record de-duplication which makes it different from other clustering tasks is the applicability of `human supervision'. For example, given two records from a medical database or two papers from DBLP or two citizens from a census data, it is fairly easy for a human to identify whether these records refer to the same physical entity. The framework of correlation clustering does not take this into account. Most prevalent approaches for data de-duplication are based on designing a similarity measure (or distance) over the records, such that records that are highly similar according to that measure are likely to be duplicate and records that measure as significantly dissimilar are likely to represent different entities. In other words, the edge labels are `close to' the underlying ground truth clustering. This is another aspect of data de-duplication which the current correlation clustering framework does not take into account. 
  
In this paper we offer a formal modelling of such record de-duplication tasks. Our framework is the same as correlation clustering but with the added \textit{promise} that the input graph edges $E$ is `close to' the optimal correlation clustering of the given dataset. We analyse the computational complexity of this problem and show that even under strong promise, correlation clustering is NP-Hard. Moreover, the problem remains NP-Hard even when we are allowed to make queries to a human expert (or an \textit{oracle}) as long as the number of queries is not too large (less than the number of points in the dataset). 

Given these negative results, we propose a \textit{restricted} variant of correlation clustering. Here, instead of finding the best clustering from the class of all possible clusterings, the learning algorithm has to choose the best clustering from a given class $\mc F$ of clusterings. We offer an algorithmic approach (which uses the help of an oracle) with success guarantees for the restricted version. The `success guarantee' depends on the complexity of the class $\mc F$ (measured by $\vcdim(\mc F)$) as well as the `closeness' of the metric $d$ to the target clustering.  

\subsection{Related Work}
The most relevant work is the framework of correlation clustering developed by \cite{bansal2004correlation} that we discussed in the previous section. Other variations of correlation clustering have been considered. For example \cite{demaine2006correlation}, consider a problem where the edges can be labelled by a real number instead of just $0$ or $1$. Edges with large positive weights encourage those vertices to be in the same cluster while edges with large negative weights encourage those points to be in different clusters. They showed that the problem is NP-Hard and gave a $O(\log n)$ approximation to the weighted correlation clustering problem. \cite{charikar2005clustering} made several contributions to the correlation clustering problem. For the problem of minimizing the correlation clustering loss (for unweighted complete graphs), they gave an algorithm with factor $4$ approximation. They also proved that the minimization problem is APX-Hard. 

More recently, \cite{ailon2018approximate} considered the problem of correlation clustering in the presence of an oracle. If the number of clusters $k$ is known, they proposed an algorithm which makes $O(k^{14} \log n)$ queries to the oracle and finds a $(1+\epsilon)$-approximation to the correlation clustering problem. They showed that the problem is NP-Hard to approximate with $o\big(\frac{k}{poly \log k}\big)$ queries to an oracle. In this work, we obtain similar results for the {promise correlation clustering} problem.

Supervision in clustering has been addressed before. For example, \cite{kulis2009semi,basu2004probabilistic,basu2002semi} considered {\em link/don't-link} constraints. This is a form of non-interactive clustering where the algorithm gets as input a list of pairs which should be in the same cluster and a list pairs which should be in different clusters. \cite{balcan2008clustering} developed a framework of interactive clustering where the supervision is provided in the form of {\em split/merge} queries. The algorithm gives the current clustering to the oracle. The oracle responds by telling the which clusters to merge and which clusters to split. 

In this work, we use the framework of same-cluster queries developed by \cite{ashtiani2016clustering}. At any given instant, the clustering algorithm asks the same-cluster oracle about two points in the dataset. The oracle replies by answering either `yes' or `no' depending upon whether the two points lie in the same or different clusters. 

On de-duplication side, most prevalent are approaches that are based on designing a similarity measure (or distance) over the records, such that records that are highly similar according to that measure are likely to be duplicates and records that measure as significantly dissimilar are likely to represent different entities. For example, to handle duplicate records created due typographical mistakes, many character-based similarity metrics have been considered. Examples of such metrics include the edit or levenshtein distance \cite{levenshtein1966binary}, smith-waterman distance \cite{waterman1981identification} and jaro distance metric \cite{jaro1980unimatch}. Token-based similarity metrics try to handle rearrangement of words, for example \cite{monge1996field} and \cite{cohen1998integration}. Other techniques include phonetic-based metrics and numerical metrics (to handle numeric data). A nice overview of these methods can be found in \cite{elmagarmid2007duplicate}. 
   
While the above approaches relied on designing a good similarity metric, some works try to `learn' the distance function from a labelled training dataset of pairs of records. Examples of such works include \cite{cochinwala2001efficient} and \cite{bilenko2003adaptive}. Clustering for de-duplication has been mostly addressed in application oriented works. \cite{hernandez1995merge} assumes that the duplicate records are transitive. The clustering problem now reduces to finding the connected components in a graph. 

\subsection{Outline}
Section \ref{section:problemFormulation} introduces the relevant notation and definitions. In Section \ref{section:PCC}, we introduce our framework of Promise Correlation Clustering. In Section \ref{section:PCCNPHard}, we prove that PCC is NP-Hard. In Section \ref{section:PCCNPHardOracle} we prove that PCC is NP-Hard even under the presence of an oracle. In Section \ref{section:RCC}, we introduce our framework of Restricted Correlation Clustering (RCC). In Sections \ref{section:samplingNegative} and \ref{section:samplingPositive} we describe procedures for sampling different-cluster (negative) and same-cluster (positive) pairs. In Section \ref{section:sampleAndQueryComplexity}, we describe our semi-supervised algorithm for solving the RCC problem. We prove an upper bound on the number of labelled samples required to guarantee the success of our algorithm. We also upper bound the number of queries made to the same-cluster oracle. Section \ref{section:conclusion} concludes our work. All the missing proofs can be found in the supplementary section. 

\section{Preliminaries}
\label{section:problemFormulation}
Given a finite domain $X$. A clustering $C$ of the set $X$ is a partition of the set $X$ into $k$ disjoint subsets, that is, $C = \{C_1, \ldots, C_k\}$. Denote by $m(C) = \max C_i$. Define $X^{[2]} = \{(x, y) : x \neq y\}$. In this paper, we view a clustering as a binary-valued function over the pairs of instances. That is, $C: X^{[2]} \rightarrow \{0, 1\}$ and $C(x, y) = 1$ if and only if $x, y$ are in the same $C$ cluster. 

Given $G = (X, E)$, define $d_E(x, y) = 0$ if there exists an edge between $x, y$ and $d_E(x, y) = 1$ otherwise. 

\begin{definition}[Correlation clustering for deduplication]\cite{bansal2004correlation}
Given $G = (X, E)$, find a clustering $C$ which minimizes 
\begin{align*}
  &L_{d_E}(C) = NL_{d_E}(C) + PL_{d_E}(C), \text{ where}\\
  &NL_{d_E}(C) = |\{(x, y): C(x, y) = 1 \text{ and } d_E(x, y) = 0\}|,\\ 
  &PL_{d_E}(C) = |\{(x, y): C(x, y) = 0 \text{ and } d_E(x, y) = 1\}| \numberthis\label{eqn:correlationLoss}
\end{align*}
$L_{d_E}(C)$ is also referred to as the correlation loss. A weighted version of the loss function places weights of $w_1$ and $w_2$ on the two terms and is defined as 
\begin{align}
  &L_{d_E}^{w_1, w_2}(C) = w_1 NL_{d_E}(C) + w_2 PL_{d_E}(C)\label{eqn:weightedCorrelationLoss}
\end{align}
\end{definition}

\begin{definition}[Informative metric]
\label{defn:informativeMetric}
Given $(X, d)$, a clustering $C^*$ and a parameter $\lambda$. We say that the metric $d$ is $(\alpha, \beta)$-informative w.r.t $C^*$ and $\lambda$ if
\begin{align}
	&\underset{(x, y) \sim U^2}{\mb P}\enspace \big[d(x, y) > \lambda \enspace|\enspace C^*(x, y) = 1\big] \enspace \le \enspace \alpha \label{eqn:alphaInformative}\\
	&\underset{(x, y) \sim U^2}{\mb P}\enspace \big[C^*(x, y) = 1 \enspace|\enspace d(x, y) \le \lambda \big] \enspace \ge \enspace \beta \label{eqn:betaInformative}
\end{align}
Here $U^2$ is the uniform distribution over $X^{[2]}$. 
\end{definition} 
This definition says that most of the same-cluster (or positive) pairs are such that the distance between them is atmost $\lambda$. Also, atleast a $\beta$ fraction of all pairs with distance $\le \lambda$ belong to the same cluster. 

To incorporate supervision into the clustering problem, we allow an algorithm to make \textit{same-cluster} queries to a $C^*$-oracle defined below.
\begin{definition}[ Same-cluster oracle \cite{ashtiani2016clustering}]
Given $X$. A same-cluster $C^*$-oracle receives a pair $x, y \in X$ as input and outputs $1$ if $x, y$ belong to the same-cluster according to $C^*$. Otherwise, it outputs $0$. 
\end{definition}

In the next section, we introduce our framework of \textit{promise correlation clustering} and discuss the computational complexity of the problem both in the absence and presence of an oracle.   

\section{Promise Correlation Clustering}
\label{section:PCC}

\begin{definition}[Promise correlation clustering (PCC)]
\label{defn:promiseCorrClustering}
Given a clustering instance $G = (X, E)$. Let $C^*$ be such that
\begin{align}
  &C^* = \argmin_{C \in \mc F} \enspace L_{d_E}(C) \label{eqn:promiseCorrLoss}
\end{align}
where $\mc F$ is the set of all possible clusterings $C$ such that $m(C) \le M$. Given that $d_E$ is $(\alpha, \beta)$-informative. Find the clustering $C^*$. 
\end{definition}
When the edges $E$ correspond to a clustering $C$ then $\beta = 1$ and $\mu = 0$. We show in the subsequent sections that even for this `restricted' class of clusterings (when the size of the maximum cluster is atmost a constant $M$) and given the prior knowledge, PCC is still NP-Hard. Furthermore, PCC is NP-Hard even when we are allowed to make $o(|X|)$ queries to a $C^*$-oracle. 

\subsection{PCC is NP-Hard}
\label{section:PCCNPHard}
\begin{figure*}[!ht]
	\centering
	\begin{tikzpicture}
	  \node[circle,draw,minimum size=1mm,,fill=black,label=below:$x_{i1}$] at (0, 0){};
	  \node[circle,draw,fill=black] at (0, 2){};
	  \node[circle,draw,fill=black] at (-1, 1){};
	  \node[circle,draw,fill=black] at (1, 1){};
	  \node[circle,draw,fill=black] at (0, 3){};

	  \node[circle,draw,fill=black] at (0, 5){};
	  \node[circle,draw,fill=black] at (-1, 4){};
	  \node[circle,draw,fill=black] at (1, 4){};
	  \node[circle,draw,fill=black] at (0, 6){};
			
	  \node[circle,draw,fill=black] at (0, 9){};
	  \node[circle,draw,fill=black] at (0, 11){};
	  \node[circle,draw,fill=black] at (-1, 10){};
	  \node[circle,draw,fill=black] at (1, 10){};
	  \node[circle,draw,fill=black] at (0, 12.5){};
	
	  \draw[blue] (0, 0) -- (0, 2);
	  \draw[blue] (0, 0) -- (-1, 1);
	  \draw[blue] (0, 0) -- (1, 1);
	  \draw (1, 1) -- (-1, 1);
	  \draw (1, 1) -- (0, 2);
	  \draw (-1, 1) -- (0, 2);
	  \draw[red] (1, 1) -- (0, 3);
	  \draw[red] (-1, 1) -- (0, 3);
	  \draw[red] (0, 3) -- (0, 2);
		
	  \draw[blue] (0, 3) -- (0, 5);
	  \draw[blue] (0, 3) -- (-1, 4);
	  \draw[blue] (0, 3) -- (1, 4);
	  \draw (1, 4) -- (-1, 4);
	  \draw (1, 4) -- (0, 5);
	  \draw (-1, 4) -- (0, 5);
	  \draw[red] (1, 4) -- (0, 6);
	  \draw[red] (-1, 4) -- (0, 6);
	  \draw[red] (0, 6) -- (0, 5);
	
	  \draw[blue] (0, 9) -- (0, 11);
	  \draw[blue] (0, 9) -- (-1, 10);
	  \draw[blue] (0, 9) -- (1, 10);
	  \draw (1, 10) -- (-1, 10);
	  \draw (1, 10) -- (0, 11);
	  \draw (-1, 10) -- (0, 11);
	  \draw[red] (1, 10) -- (0, 12.5);
	  \draw[red] (-1, 10) -- (0, 12.5);
	  \draw[red] (0, 12.5) -- (0, 11);
	  	
	  \node[circle,draw,minimum size=1mm,,fill=black,label=below:$x_{i2}$] at (3, 0){};
	  \node[circle,draw,fill=black] at (3, 2){};
	  \node[circle,draw,fill=black] at (2.5, 1){};
	  \node[circle,draw,fill=black] at (3.5, 1){};
	  \node[circle,draw,fill=black] at (3, 3){};
	
	  \node[circle,draw,fill=black] at (3, 5){};
	  \node[circle,draw,fill=black] at (2.5, 4){};
	  \node[circle,draw,fill=black] at (3.5, 4){};
	  \node[circle,draw,fill=black] at (3, 6){};
		
	  \node[circle,draw,fill=black] at (3, 9){};
	  \node[circle,draw,fill=black] at (3, 11){};
	  \node[circle,draw,fill=black] at (2.5, 10){};
	  \node[circle,draw,fill=black] at (3.5, 10){};
	  \node[circle,draw,fill=black] at (3, 12){};
	  		
	  \draw[blue] (3, 0) -- (3, 2);
	  \draw[blue] (3, 0) -- (2.5, 1);
	  \draw[blue] (3, 0) -- (3.5, 1);
	  \draw (3.5, 1) -- (2.5, 1);
	  \draw (3.5, 1) -- (3, 2);
	  \draw (2.5, 1) -- (3, 2);
	  \draw[red] (2.5, 1) -- (3, 3);
	  \draw[red] (3.5, 1) -- (3, 3);
	  \draw[red] (3, 3) -- (3, 2);
	  
	  \draw[blue] (3, 3) -- (3, 5);
	  \draw[blue] (3, 3) -- (2.5, 4);
	  \draw[blue] (3, 3) -- (3.5, 4);
	  \draw (3.5, 4) -- (2.5, 4);
	  \draw (3.5, 4) -- (3, 5);
	  \draw (2.5, 4) -- (3, 5);
	  \draw[red] (3.5, 4) -- (3, 6);
	  \draw[red] (2.5, 4) -- (3, 6);
	  \draw[red] (3, 6) -- (3, 5);
	  
	  \draw[blue] (3, 9) -- (3, 11);
	  \draw[blue] (3, 9) -- (2.5, 10);
	  \draw[blue] (3, 9) -- (3.5, 10);
	  \draw (3.5, 10) -- (2.5, 10);
	  \draw (3.5, 10) -- (3, 11);
	  \draw (2.5, 10) -- (3, 11);
	  \draw[red] (3.5, 10) -- (3, 12);
	  \draw[red] (2.5, 10) -- (3, 12);
	  \draw[red] (3, 12) -- (3, 11);
	  
	  \node[circle,draw,minimum size=1mm,,fill=black,label=below:$x_{i3}$] at (5, 0){};
	  \node[circle,draw,fill=black] at (5, 2){};
	  \node[circle,draw,fill=black] at (4.5, 1){};
	  \node[circle,draw,fill=black] at (5.5, 1){};
	  \node[circle,draw,fill=black] at (5, 3){};
	
	  \node[circle,draw,fill=black] at (5, 5){};
	  \node[circle,draw,fill=black] at (4.5, 4){};
	  \node[circle,draw,fill=black] at (5.5, 4){};
	  \node[circle,draw,fill=black] at (5, 6){};
		
	  \node[circle,draw,fill=black] at (5, 9){};
	  \node[circle,draw,fill=black] at (5, 11){};
	  \node[circle,draw,fill=black] at (4.5, 10){};
	  \node[circle,draw,fill=black] at (5.5, 10){};
	  \node[circle,draw,fill=black] at (5, 12){};
	  		
	  \draw[blue] (5, 0) -- (5, 2);
	  \draw[blue] (5, 0) -- (4.5, 1);
	  \draw[blue] (5, 0) -- (5.5, 1);
	  \draw (5.5, 1) -- (4.5, 1);
	  \draw (5.5, 1) -- (5, 2);
	  \draw (4.5, 1) -- (5, 2);
	  \draw[red] (5.5, 1) -- (5, 3);
	  \draw[red] (4.5, 1) -- (5, 3);
	  \draw[red] (5, 3) -- (5, 2);
	  
	  \draw[blue] (5, 3) -- (5, 5);
	  \draw[blue] (5, 3) -- (4.5, 4);
	  \draw[blue] (5, 3) -- (5.5, 4);
	  \draw (5.5, 4) -- (4.5, 4);
	  \draw (5.5, 4) -- (5, 5);
	  \draw (4.5, 4) -- (5, 5);
	  \draw[red] (5.5, 4) -- (5, 6);
	  \draw[red] (4.5, 4) -- (5, 6);
	  \draw[red] (5, 6) -- (5, 5);
	  
	  \draw[blue] (5, 9) -- (5, 11);
	  \draw[blue] (5, 9) -- (4.5, 10);
	  \draw[blue] (5, 9) -- (5.5, 10);
	  \draw (5.5, 10) -- (4.5, 10);
	  \draw (5.5, 10) -- (5, 11);
	  \draw (4.5, 10) -- (5, 11);
	  \draw[red] (5.5, 10) -- (5, 12);
	  \draw[red] (4.5, 10) -- (5, 12);
	  \draw[red] (5, 12) -- (5, 11);
	  
	  \node[circle,draw,minimum size=1mm,label=below:$r_{1}$,fill=black] at (8, 0){};
	  \node[circle,draw,fill=black] at (8, 2){};
	  \node[circle,draw,fill=black] at (7, 1){};
	  \node[circle,draw,fill=black] at (9, 1){};
	  \node[circle,draw,fill=black] at (8, 3){};

	  \node[circle,draw,fill=black] at (8, 5){};
	  \node[circle,draw,fill=black] at (7, 4){};
	  \node[circle,draw,fill=black] at (9, 4){};
	  \node[circle,draw,fill=black] at (8, 6){};
			
	  \node[circle,draw,fill=black] at (8, 9){};
	  \node[circle,draw,fill=black] at (8, 11){};
	  \node[circle,draw,fill=black] at (7, 10){};
	  \node[circle,draw,fill=black] at (9, 10){};
	  \node[circle,draw,fill=black] at (8, 12.5){};
	
	  \draw[blue] (8, 0) -- (8, 2);
	  \draw[blue] (8, 0) -- (7, 1);
	  \draw[blue] (8, 0) -- (9, 1);
	  \draw (9, 1) -- (7, 1);
	  \draw (9, 1) -- (8, 2);
	  \draw (7, 1) -- (8, 2);
	  \draw[red] (9, 1) -- (8, 3);
	  \draw[red] (7, 1) -- (8, 3);
	  \draw[red] (8, 3) -- (8, 2);
		
	  \draw[blue] (8, 3) -- (8, 5);
	  \draw[blue] (8, 3) -- (7, 4);
	  \draw[blue] (8, 3) -- (9, 4);
	  \draw (9, 4) -- (7, 4);
	  \draw (9, 4) -- (8, 5);
	  \draw (7, 4) -- (8, 5);
	  \draw[red] (9, 4) -- (8, 6);
	  \draw[red] (7, 4) -- (8, 6);
	  \draw[red] (8, 6) -- (8, 5);
	
	  \draw[blue] (8, 9) -- (8, 11);
	  \draw[blue] (8, 9) -- (7, 10);
	  \draw[blue] (8, 9) -- (9, 10);
	  \draw (9, 10) -- (7, 10);
	  \draw (9, 10) -- (8, 11);
	  \draw (7, 10) -- (8, 11);
	  \draw[red] (9, 10) -- (8, 12.5);
	  \draw[red] (7, 10) -- (8, 12.5);
	  \draw[red] (8, 12.5) -- (8, 11);
	
	  \draw[blue] (0, 12.5) -- (8, 12.5);
	  \draw[blue] (0, 12.5) -- (5, 12);
	  \draw[blue] (0, 12.5) -- (3, 12);
	  \draw[blue] (3, 12) -- (5, 12);
	  \draw[blue] (3, 12) -- (8, 12.5);
	  \draw[blue] (5, 12) -- (8, 12.5);
	  
	  \path (0, 9) -- (0, 6) node [black, font=\Huge, midway, sloped] {$\dots$};
	  \path (3, 9) -- (3, 6) node [black, font=\Huge, midway, sloped] {$\dots$};
	  \path (5, 9) -- (5, 6) node [black, font=\Huge, midway, sloped] {$\dots$};
	  \path (8, 9) -- (8, 6) node [black, font=\Huge, midway, sloped] {$\dots$};
	  
	  \draw [decorate,decoration={brace,amplitude=10pt},xshift=-4pt,yshift=0pt]
(-1,0) -- (-1,3) node [black,midway,xshift=-0.6cm] 
{\footnotesize $B_1$};
	  \draw [decorate,decoration={brace,amplitude=10pt},xshift=-4pt,yshift=0pt]
(-1,3) -- (-1,6) node [black,midway,xshift=-0.6cm] 
{\footnotesize $B_2$};
	  \draw [decorate,decoration={brace,amplitude=10pt},xshift=-4pt,yshift=0pt]
(-1,9) -- (-1,12.5) node [black,midway,xshift=-0.6cm] 
{\footnotesize $B_t$};

	\end{tikzpicture}
\caption{Part of graph $G$ constructed for the subset $S_i =  \{x_{i1}, x_{i2}, x_{i3}\}$. The graph is constructed by local replacement when for $p = 4$. If $S_i$ is included in the exact cover then the edges colored black and the edges colored blue represent the corresponding clustering of this part of the graph $G$. If $S_i$ is not included in the exact cover then the edges colored red and the edges colored black represent the clustering of this part of the graph.}
\label{fig:X3CNPHard}
\end{figure*}
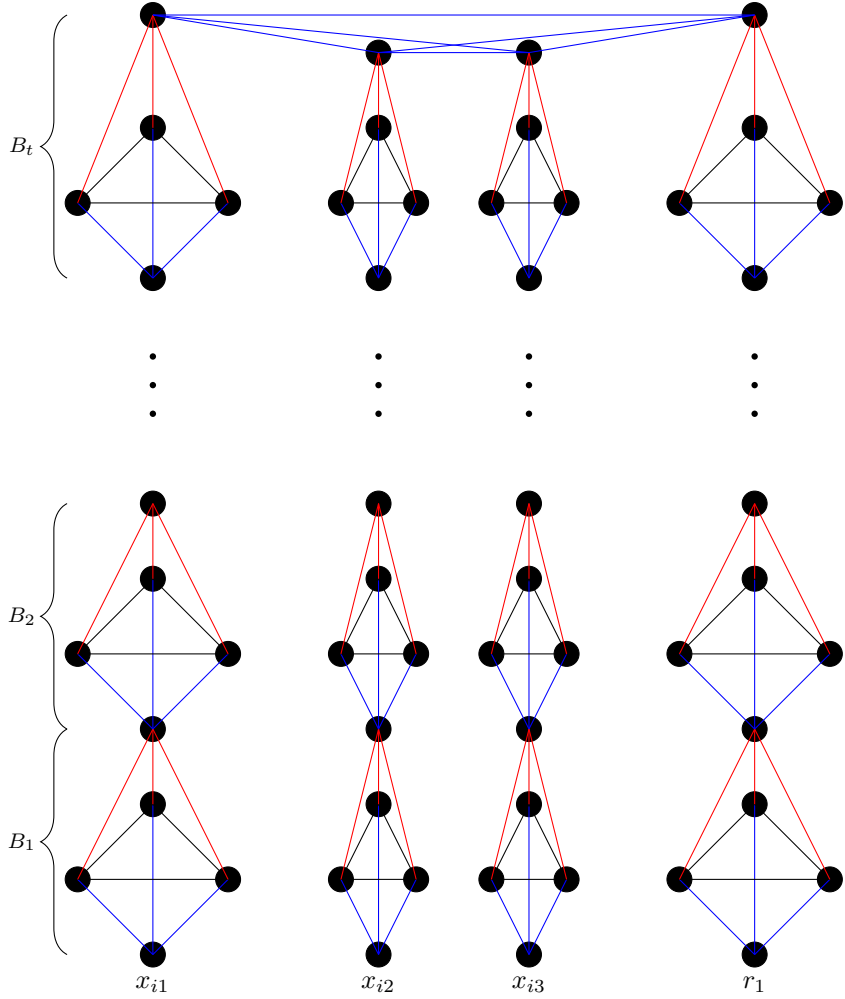

\begin{theorem}
Finding the optimal solution to the Promise Correlation Clustering problem is NP-Hard for all $M \ge 3$ and for $\alpha = 0$ and $\beta = \frac{1}{2}$.  
\end{theorem}

\noindent To prove the result, we will use a reduction from exact cover by $3$-sets problem which is known to be NP-Hard.

(X3C) Given a universe of elements $U = \{x_1, \ldots, x_{3q}\}$ and a collections of subsets $S = \{S_1, \ldots, S_m\}$. Each $S_i \subset U$ and contains exactly three elements. Does there exist $S' \subseteq S$ such that each element of $U$ occurs exactly once in $S'$?

This decision problem is known to be NP-Hard. We will now reduce an instance of X3C to the promise correlation clustering problem. For each three set $S_i = \{x_{i1}, x_{i2}, x_{i3}\}$, we construct a replacement gadget as described in Fig. \ref{fig:X3CNPHard}. The gadget is similar to the one used in the proof of partition into triangles problem. However, instead of triangles the graph is `made of' cliques of size $p$.

Given an instance of X3C, we construct $G = (V, E)$ using local replacement described in Fig. \ref{fig:X3CNPHard}. Let $A$ be an algorithm which solves the promise problem described in Eqn. \ref{eqn:promiseCorrLoss}. Then, we can use this algorithm to decide exact cover by three sets as follows.

If $A$ outputs a clustering $C$ such that all the clusters have size exactly $p$ and $E_C$ makes no negative errors w.r.t $E$ (that is $\mu(E_C) = 0$) then output YES. Otherwise, output NO. Next, we will prove that this procedure decides X3C. 

Let there exists an exact cover for the X3C instance. Let $C$ be the clustering corresponding to the exact cover. That is, the edges colored blue and black correspond to this clustering and the corresponding vertices are in the same cluster (Fig. \ref{fig:X3CNPHard}). Note that this clustering makes no negative errors. Furthermore, each point is in a cluster of size exactly $p$. Thus, the positive error corresponding to any vertex is the degree of that vertex minus $p-1$. Since, the size of a cluster is atmost $p$, this is the minimum possible positive error for any vertex. Hence, any other clustering strictly makes more positive errors than $C$. 

It is easy to see from the construction that if $A$ finds a clustering which has no negative errors and all the clusters have size $p$, then this corresponds to exact cover of the X3C instance and hence we output YES. If this does not happen then there does not exist any exact cover for $(U, S)$. This is because if there was an exact cover then the corresponding clustering would satisfy our condition. Thus, $A$ decides X3C. Since, X3C is NP-Hard, no polynomial time algorithm $A$ exists unless $P = NP$.

In the construction, for each clause, we have $p^2 t + (p - 3)$ vertices and a vertex for each of the variables. Therefore, $|V| = m (p^2 t + (p-3)) + 3q$ and $|E| = pt({p \choose 2}+p-1) + {p \choose 2}$.  Consider a clustering $C$ which places all the $x_i$'s  and $r_i$'s in singleton clusters and places rest of the points in clusters of size $p$. For $t \ge 2$,
\begin{align*}
  \beta &= \frac{pt{p \choose 2}}{pt({p \choose 2}+p-1) + {p \choose 2}} = \frac{1}{1 + \frac{2}{p} + \frac{1}{pt}} > \frac{1}{2} \text{ and }\alpha = 0
\end{align*} 

\subsection{Hardness of PCC in the presence of an oracle}
\label{section:PCCNPHardOracle}
In the previous sections, we have shown that the PCC problem is NP-Hard without queries. It is trivial to see that by making $\beta |X|$ queries to the same-cluster oracle allows us to solve (in polynomial time) the Promise Correlation Clustering problem for all $M$ and $\alpha = 0$. In this section, we prove that the linear dependence on $n = |X|$ is tight. We prove that if the exponential time hypothesis (ETH) holds then any algorithm that runs in polynomial time makes atleast $\Omega(n)$ same-cluster queries.

\begin{theorem}
Given that the Exponential Time Hypothesis (ETH) holds then any algorithm for the Promise Correlation Clustering problem  that runs in polynomial time makes $\Omega(|X|)$ same-cluster queries for all $M \ge 3$ and for $\alpha = 0$ and $\beta = \frac{1}{2}$. 
\end{theorem}

\noindent Below, we give a proof sketch but a detailed proof is in the supplementary material. The exponential time hypothesis says that any solver for $3$-SAT runs in $2^{o(m)}$ time (where $m$ is the number of clauses in the $3$-SAT formula). We use a reduction from $3$-SAT to 3DM to X3C to show that the exact cover by 3-sets (X3C) problem also can't be solved in $2^{o(m)}$ time (if ETH holds). Then, using the reduction from the previous section implies that PCC also can't be solved in $2^{o(n)}$ time. Thus, any query based algorithm for PCC needs to make atleast $\Omega(n)$ queries where $n = |X|$ is the number of vertices in the graph. 

\begin{definition}[3-SAT].\\
Input: A boolean formulae $\phi$ in 3CNF with $n$ literals and $m$ clauses. Each clause has exactly three literals. \\
Output: YES if $\phi$ is satisfiable, NO otherwise. 
\end{definition}

\noindent\textbf{Exponential Time Hypothesis}\\
There does not exist an algorithm which decides 3-SAT  and runs in $2^{o(m)}$ time.

\noindent To prove that (X3C) is NP-Hard, the standard We will reduce 3-SAT to 3-dimensional matching problem. 3DM is already known to be NP-Hard. However, the standard reduction of 3-SAT to 3DM constructs a set with number of matchings in $\Theta(m^2 n^2)$. Hence, using the standard reduction, the exponential time hypothesis would imply there does not exist an algorithm for 3DM which runs in $\Omega(m^\frac{1}{4})$. Our reduction is based on the standard reduction. However, we make some clever optimizations especially in the way we encode the clauses. This improves the lower bound to $\Omega(m)$.

Using the above result, we immediately get an $\Omega(2^m)$ lower bound on the run-time of X3C. Now, using the same reduction of X3C to PCC as in Section \ref{subsection:PCCNPHard}, gives the same lower bound of $\Omega(n)$ on the running time of PCC. 

For the sake of contradiction, let us assume that there exists an algorithm which solves PCC in polynomial time by making $o(n)$ same-cluster queries ($n$ is the number of vertices). Then by simulating all possible answers for the oracle, we get a non-query algorithm which solves PCC in $2^{o(n)}$. Hence, no such query algorithm exists. 

\section{Restricted Correlation Clustering}
\label{section:RCC}
The results in the previous section show that even under strong promise, correlation clustering is still NP-Hard. Furthermore, it is hard even when given access to an oracle. This motivates us to consider a restricted version of the problem. 

Note that in Defn. \ref{defn:promiseCorrClustering}, the optimization problem was over the set of all possible clusterings $\mc F$ (with a restriction on the maximum cluster size). In this section, we restrict $\mc F$ to be a finite class of clusterings. That is, $\mc F' = \{T_1, \ldots, T_r, C_1, \ldots, C_s\}$ with the understanding that each $T_i$ (a hierarchical clustering tree of $X$) is a collection of clusterings represented by the prunings of the tree. Now, we consider two versions of correlation clustering on this restricted family. The first is to find a clustering $C \in \mc F'$ which correlates `as much as possible' with the given graph $G = (X, E)$. More formally, given $G = (X, E)$ find $\hat C \in \mc F'$ such that 
\begin{align}
\hat C = \argmin_{C \in \mc F'} \enspace L_E(C) \label{eqn:RCCTrivial}
\end{align}
Eqn. \ref{eqn:RCCTrivial} can be solved by going over the list of clusterings and trees (in a bottom-up fashion) and in polynomial time finding $\hat C$ which is `closest' to $E$. In the second version, the goal is to find a clustering $\hat C \in \mc F'$ which correlates as much as possible to an unknown target clustering $C^*$ which may or may not be in the set $\mc F'$. However, the algorithm has access to a $C^*$-oracle. For the rest of this paper, we will focus on the second version which we call \textit{restricted correlation clustering}. 

\begin{definition}[Restricted correlation clustering (RCC)]
\label{defn:rcc}
Given a clustering instance $(X, d)$. Let $C^*$ be an unknown target clustering of $X$ and weights $w_1, w_2$. Let $d_{C^*}: X^{[2]} \rightarrow \{0, 1\}$ be defined as $d_{C^*}(x, y) = 0$ if $x, y$ are in the same $C^*$ cluster and $1$ otherwise. Find $\hat C \in \mc F'$ such that 
\begin{align}
\hat C = \argmin_{C \in \mc F'} \enspace L^{w_1, w_2}_{d_{C^*}}(C)\label{eqn:RCCMain}
\end{align}
where $\mc F' = \{T_1, \ldots, T_r, C_1, \ldots, C_s\}$. $T_i$ is a hierarchical clustering tree and $C_i$ is a clustering of $X$.
\end{definition}

To solve the RCC problem, we adopt the following strategy. We use a procedure (call it $\mc P_0$) to sample negative (or different cluster) pairs and another procedure (call it $\mc P_1$) to sample positive (or same-cluster) pairs. Both the sampling procedures use the help of the $C^*$-oracle. We then evaluate each of the clusterings in $\mc F'$ on our sample $S$ and choose the clustering which has minimum loss. We prove that the loss of the clustering $\hat C$ obtained using this procedure is close to the loss of $\hat C^*$ (the clustering with minimum loss in $\mc F'$) .

We first discuss how to sample the positive and negative pairs. Then, we discuss the sample complexity of our approach. That is, the number of positive and negative pairs (or $|S|$) needed to guarantee that the loss of $\hat C$ is close to that of $\hat C^*$. Before we proceed, lets introduce the following definitions which will be useful in the subsequent sections. 
\begin{definition}[Restricted distributions]
\label{defn:restrictedDistribution}
Given $X$ and a target clustering $C^*$. Define $X^{[2]+} = \{(x, y) \in X^{[2]} : C^*(x, y) = 1\}$ and $X^{[2]-} = \{(x, y) \in X^{[2]} : C^*(x, y) = 0\}$. We define $P^+$ as the uniform distribution over $X^{[2]+}$ and $P^-$ as the uniform distribution over $X^{[2]-}$. 
\end{definition}

The sampling procedure $\mc P_0$ will try to approximate $P^-$ while $\mc P_1$ will approximate $P^+$. 

\begin{definition}[$\gamma$-skewed] 
\label{defn:gammaskewed} 
Given $X$ and a $C^*$-oracle. We say that $X$ is $\gamma$-skewed w.r.t $C^*$ if 
\begin{align*}
	&\underset{(x, y) \sim U^2}{\mb P}\enspace \big[\enspace C^*(x, y) = 1 \big] \enspace \le \enspace \gamma
\end{align*}
\end{definition}
The above definition formalizes the statement that most of the pairs of points belong to different clusters. 
 
\subsection{Sampling negative pairs}
\label{section:samplingNegative}
Assume our input $X$ is $\gamma$-skewed. Thus, if we choose a pair uniformly at random, then it is `highly likely' to be a negative pair. Alg. \ref{alg:weightedNegPairs} describes our sampling procedure.

\RestyleAlgo{ruled}
\SetAlgoNoLine
\LinesNumbered
\SetNlSkip{-0.4em}
\begin{algorithm}
\label{alg:weightedNegPairs}
\caption{Procedure $\mc P_0$ for negative pairs}

\Indp\KwIn{A set $X$ and a $C^*$-oracle.}
\KwOut{One pair $(x, y) \in X^{[2]}$ such that $C^*(x, y) = 0$}

\vspace{0.1in} 
\While{TRUE}{
  Sample $(x, y)$ using $U^2$\\
  \If {$ C^*(x, y) = 0$}{\label{algLine:oracleP1}
  		Output $(x, y)$
	}
}
\end{algorithm}

\begin{restatable}{lem}{weightedNegUniform}
\label{lemma:weightedNegUniform}
Given $X$ and a $C^*$-oracle. The procedure $\mc P_0$ samples a pair $(x, y)$ according to the distribution $P^-$.
\end{restatable}
\begin{proof}
The probability that a negative pair is sampled during a trial is $U^2(X^{[2]-}) =: q$. Fix a negative pair $(x, y)$ and let $U^2(x, y) = p$. Hence, the probability that the pair $(x, y)$ is sampled $= p + (1-q)p + (1-q)^2p + \ldots = p\sum_{i=0}^\infty (1-q)^i = \frac{p}{q} = \frac{U^2(x, y)}{U^2(X^{[2]-})} = P^-(x, y)$.
\end{proof}

\noindent Note that to sample one negative pair, procedure $\mc P_0$ might need to ask more than one same-cluster query. However, since our input $X$ is $\gamma$-skewed, we `expect' the number of `extra' queries to be `small'. 

\begin{restatable}{lem}{negQueries}
\label{lemma:negQueries}Given set $X$ and a $C^*$-oracle. Let $X$ be $\gamma$-skewed and Let $q$ be the number of same-cluster queries made by $\mc P_0$ to the $C^*$-oracle. Then, $\mb E[q] \le \frac{1}{1-\gamma}$.
\end{restatable}
\begin{proof}
Let $p$ denote the probability that a negative pair is sampled during an iteration. We know that $p \ge (1 - \gamma)$. Let $q$ be a random variable denoting the number of iterations (or trials) before a negative pair is sampled. Then, $q$ is a geometric random variable. $\mb E[q] = \frac{1}{p} \le \frac{1}{1-\gamma}$.
\end{proof}

Lemma \ref{lemma:negQueries} shows that for $\gamma < \frac{1}{2}$, to sample a negative pair, procedure $\mc P_0$ makes atmost two queries to the oracle in expectation. Moreover, the number of queries is tight around the mean. Note that this sampling strategy is not useful for positive pairs. This is because the fraction of positive pairs in the dataset is small. Hence, to sample a single positive pair we would need to make `many' same-cluster queries. 

\subsection{Sampling positive pairs}
\label{section:samplingPositive}

Given a clustering instance $(X, d)$. Assume that the metric $d$ is $(\alpha, \beta)$-informative w.r.t target $C^*$ and parameter $\lambda$. This means that `most' of the positive pairs are within distance $\lambda$. Our sampling strategy is to ``construct" a set $K = \{(x, y) \in X^2: d(x, y) \le \lambda\}$ and then sample uniformly from this set. We will prove that this procedure approximates $P^+$. 

The sampling algorithm is described in Alg. \ref{alg:weightedPosPairs}. In the pre-compute stage, for all points $x$ we construct its set of `neighbours' ($S_x$). We then choose a point with probability proportional to the size of its neighbour-set and then choose the second point uniformly at random from amongst its neighbours. This guarantees that we sample uniformly from the set $K$.  

\RestyleAlgo{ruled}
\SetAlgoNoLine
\LinesNumbered
\begin{algorithm}
\label{alg:weightedPosPairs}
\caption{Sampling procedure $\mc P_{11}$ for positive pairs (general metrics)}
\Indp\KwIn{A set $X$, a $C^*$-oracle and a parameter $\lambda$.}
\KwOut{One pair $(x, y) \in X^{[2]}$ such that $\mc C^*(x, y) = 1$}
\vspace{0.1in}\textbf{Pre-compute:} For all $x \in X$, compute $S_x := \{y: d(x, y) \le \lambda\}$.\\

\vspace{0.1in} \While{TRUE}{
Sample $x \in X$ with probability $\propto |S_x|$. \label{algLine:size}\\
Sample $y$ uniformly at random from $S_x$. \label{algLine:sample}\\
\If {$\mc C^*(x, y) = 1$}{\label{algLine:oracleP21}
	Output $(x, y)$.
	}
}
\end{algorithm}

\begin{restatable}{lem}{weightedPosApproxUniform}
\label{lemma:weightedPosApproxUniform}
Given set $(X, d)$, a $C^*$-oracle and parameter $\lambda$. Let $d$ be $(\alpha, \beta)$-informative w.r.t $\lambda$ and $C^*$.  Then the sampling procedure $\mc P_{11}$ induces a distribution $T$ over $X^{[2]}$ such that for any labelling function $h$ over $X^{[2]}$ we have that $$\Big|\underset{(x, y) \sim P^+}{\mb P}\enspace \big[ h(x, y) = 0 ] - \underset{(x, y) \sim T}{\mb P}\enspace \big[ h(x, y) = 0 ]\Big|  \enspace \le \enspace 2\alpha.$$ 
\end{restatable}

\noindent Note that to sample one positive pair, procedure $\mc P_{11}$ might need to ask more than one same-cluster query. However, since the metric $d$ is $\beta$-informative, we `expect' the number of `extra' queries to be `small'. 

\begin{restatable}{lem}{posQueries}
\label{lemma:posQueries}
Given set $(X, d)$, a $C^*$-oracle and a parameter $\lambda$. Let $d$ be $\beta$-informative w.r.t $\lambda$ and let $q$ be the number of same-cluster queries made by $\mc P_{11}$ to the $C^*$-oracle. Then, $\mb E[q] \le \frac{1}{\beta}$.
\end{restatable}
\begin{proof}
Let $p$ denote the probability that a positive pair is sampled during an iteration. We know that $p \ge \beta$. Let $q$ be a random variable denoting the number of iterations (or trials) before a positive pair is sampled. Then, $q$ is a geometric random variable. $\mb E[q] = \frac{1}{p} \le \frac{1}{\beta}$.
\end{proof}

\section{Sample and query complexity of RCC}
\label{section:sampleAndQueryComplexity}
In the previous section, we developed a sampling procedure for positive and negative pairs. We showed that the procedures sample according to  distributions $T_1$ and $T_2$ which approximate $P^-$ and $P^+$ respectively. Given a class of clusterings $\mc F$, we use our distributions $T_1$ and $T_2$ to estimate the negative and positive components of the loss for each clustering $C \in \mc F$. We then choose the clustering $\hat C$ with the minimum estimated loss. Using standard VC-Dimension theory, it is easy to show that the loss of the clustering $\hat C$ is close to the loss of best clustering in $\mc F$, as long the VC-Dimension of $\mc F$ is finite. 

The loss function $L_{d_{C^*}}^{w_1, w_2}$ is the sum of the sizes of two sets. However, in this section it would be more convenient to work with bounded loss functions. Let $\gamma_0 = \underset{(x, y) \sim U^2}{\mb P}\enspace \big[\enspace C^*(x, y) = 1\big]$ and define $\mu = \frac{w_1 \gamma_0}{w_1 \gamma_0 + w_2(1-\gamma_0)}$. Then we see that minimizing Eqn. \ref{eqn:RCCMain}, is the same as minimizing
\begin{definition}[Normalized correlation loss]
\begin{align*}
  L_{C^*}(C) = & \mu \underset{(x, y) \sim P^+}{\mb P} \big[ C(x, y) = 0 ] + (1-\mu) \underset{(x, y) \sim P^-}{\mb P} \big[ C(x, y) = 1] \numberthis\label{eqn:RCCV2}
\end{align*}
\end{definition}
For the remainder of the section, we work with this formulation of the loss function. We describe this procedure in Alg. \ref{alg:ERM}. 

\RestyleAlgo{ruled}
\SetAlgoNoLine
\LinesNumbered
\SetNlSkip{-0.4em}
\begin{algorithm}[h]
\label{alg:ERM}
\caption{Empirical Risk Minimization}
\Indp\KwIn{$( X, d)$, a set of clusterings $\mc F$, a $C^*$-oracle, parameter $\lambda$ and sizes $m_+$ and $m_-$.}
\KwOut{$ C \in \mc F$}

\vspace{0.1in} 
Sample a sets $S_+$ and $S_-$ of sizes $m_+$ and $m_-$ using procedures $\mc P_{11}$ and $\mc P_0$.\\
For every $C \in \mc F$ and define 
\vspace{-0.1in}\begin{align*}
  &\hat E(C) = \frac{|\{(x, y) \in S_+: C(x, y) = 0\}|}{|S_+|}\\
  &\hat G(C) = \frac{|\{(x, y) \in S_-: C(x, y) = 0\}|}{|S_-|}
\end{align*}

Define $\hat L(h) = \mu \hat E(h) + (1-\mu)\hat G(h)$. \label{algLine:alpha}\\
Output $\argmin_{C \in \mc F} \enspace \hat L(l_{C})$ \label{algLine:ERM}
\end{algorithm}

\begin{restatable}{thrm}{sampleComplexity}
\label{thm:sampleComplexity}
Given metric space $(X, d)$, a class of clusterings $\mc F$ and a threshold parameter $\lambda$. Given $\epsilon, \delta \in (0, 1)$ and a $C^*$-oracle. Let $d$ be $(\alpha, \beta)$-informative and $X$ be $\gamma$-skewed w.r.t $\lambda$ and $C^*$. Let $\mc A$ be the ERM-based approach as described in Alg. \ref{alg:ERM} and $\hat C$ be the output of $\mc A$. If  
\begin{align}
  &m_-, m_+ \enspace \ge a\frac{\vcdim({\mc F}) + \log(\frac{2}{\delta})}{\epsilon^2} 
\end{align}
where $a$ is a global constant then with probability atleast $1-\delta$ (over the randomness in the sampling procedure), we have that $$L_{C^*}(\hat C) \enspace\le\enspace \min_{\mc C \in \mc F} L_{C^*}(\mc C) + 3\alpha + \epsilon$$
\end{restatable}

Next we show that to sample $m_+$ positive and $m_-$ negative pairs, the number of queries made to the $C^*$ is not too large. 
\begin{restatable}{thrm}{queryComplexity}[Query Complexity]
\label{thm:queryComplexity}
Let the framework be as in Thm. \ref{thm:sampleComplexity}. With probability atleast $1-\exp\big(-\frac{\nu^2m_-}{4}) - \exp\big(-\frac{\nu^2m_+}{4}\big)$ over the randomness in the sampling procedure, the number of same-cluster queries $q$ made by $\mc A$ is  
$$q \le (1+\nu)\bigg(\frac{m_-}{(1-\gamma)} + \frac{m_+}{\beta}\bigg)$$
\end{restatable}

\subsection{VC-Dimension of some common classes of clusterings}
In the previous section, we proved that the sample complexity of learning a class of clusterings $\mc F$ depends upon $\vcdim({\mc F})$. Recall that ${\mc F}$ is the class of labellings induced by the clusterings in $\mc F$. In this section, we prove upper bounds on the VC-Dimension for some common class of clusterings. 

\begin{restatable}{thrm}{VCDim}
Given a finite set $\mc X$ and a finite class $\mc F = \{C_1, \ldots, C_s\}$ of clusterings of $\mc X$.
$$\vcdim({\mc F}) \le g(s)$$ where $g(s)$ is the smallest integer $n$ such that $B_{\sqrt n} \ge s$ where $B_i$ is the $i^{th}$ bell number \cite{bell2010number}. 
\end{restatable}

\noindent Note that $B_{\sqrt n} \in o(2^{n})$. Thus, the $\vcdim$ of a list of clusterings is in $o( \log s)$. Next, we discuss another common class of clusterings, namely hierarchical clustering trees. 

\begin{definition}[Hierarchical clustering tree]
Given a set $X$. A hierarchical clustering tree $T$ is a rooted binary tree with the elements of $X$ as the leaves. 
\end{definition}

\noindent Every pruning of a hierarchical clustering tree is a clustering of the set $X$. A clustering tree contains exponentially many (in the size of $\mc X$) clusterings. Given $\mc F = \{T_1, \ldots, T_s\}$ consists of $s$ different hierarchical clustering trees, the following theorem bounds the VC-Dimension of ${\mc F}$.

\begin{restatable}{thrm}{VCDimT}
Given a finite set $\mc X$ and a finite class $\mc F = \{T_1, \ldots, T_s\}$ where each $T_i$ is a hierarchical clustering over $\mc X$. Then 
$$\vcdim({\mc F}) \le g(s)$$ where $g(s)$ is the smallest integer $n$ such that $\frac{\sqrt n!}{\lfloor \sqrt n/2 \rfloor! \enspace 2^{\lfloor \sqrt n/2 \rfloor}} \ge s $
\end{restatable}

\section{Conclusion}
\label{section:conclusion}
We introduced a promise version of correlation clustering. We proved that the promise version is NP-Hard. Furthermore, the problem is NP-Hard even when we are allowed to make $o(|X|)$ queries to a same-cluster oracle (where $X$ is the clustering instance). We then introduced a restricted version of correlation clustering. We developed a sampling procedure (with the help of the same-cluster oracle) to sample same-cluster and different-cluster pairs. We then used this procedure to solve the restricted variant.    

\bibliography{deDuplication}

\begin{thebibliography}{}

\bibitem[A000108, ]{bell2010number}
A000108, S.
\newblock The on-line encyclopedia of integer sequences.
\newblock {\em published electronically at https://oeis.org, 2010}.

\bibitem[Ailon et~al., 2018]{ailon2018approximate}
Ailon, N., Bhattacharya, A., and Jaiswal, R. (2018).
\newblock Approximate correlation clustering using same-cluster queries.
\newblock In {\em Latin American Symposium on Theoretical Informatics}, pages
  14--27. Springer.

\bibitem[Ashtiani et~al., 2016]{ashtiani2016clustering}
Ashtiani, H., Kushagra, S., and Ben-David, S. (2016).
\newblock Clustering with same-cluster queries.
\newblock In {\em Advances in neural information processing systems}, pages
  3216--3224.

\bibitem[Balcan and Blum, 2008]{balcan2008clustering}
Balcan, M.-F. and Blum, A. (2008).
\newblock Clustering with interactive feedback.
\newblock In {\em International Conference on Algorithmic Learning Theory},
  pages 316--328. Springer.

\bibitem[Bansal et~al., 2004]{bansal2004correlation}
Bansal, N., Blum, A., and Chawla, S. (2004).
\newblock Correlation clustering.
\newblock {\em Machine Learning}, 56(1-3):89--113.

\bibitem[Basu et~al., 2002]{basu2002semi}
Basu, S., Banerjee, A., and Mooney, R. (2002).
\newblock Semi-supervised clustering by seeding.
\newblock In {\em In Proceedings of 19th International Conference on Machine
  Learning (ICML-2002}. Citeseer.

\bibitem[Basu et~al., 2004]{basu2004probabilistic}
Basu, S., Bilenko, M., and Mooney, R.~J. (2004).
\newblock A probabilistic framework for semi-supervised clustering.
\newblock In {\em Proceedings of the tenth ACM SIGKDD international conference
  on Knowledge discovery and data mining}, pages 59--68. ACM.

\bibitem[Bilenko et~al., 2003]{bilenko2003adaptive}
Bilenko, M., Mooney, R., Cohen, W., Ravikumar, P., and Fienberg, S. (2003).
\newblock Adaptive name matching in information integration.
\newblock {\em IEEE Intelligent Systems}, 18(5):16--23.

\bibitem[Blumer et~al., 1989]{blumer1989learnability}
Blumer, A., Ehrenfeucht, A., Haussler, D., and Warmuth, M.~K. (1989).
\newblock Learnability and the vapnik-chervonenkis dimension.
\newblock {\em Journal of the ACM (JACM)}, 36(4):929--965.

\bibitem[Brown, 2011]{brown2011wasted}
Brown, D.~G. (2011).
\newblock How i wasted too long finding a concentration inequality for sums of
  geometric variables.
\newblock {\em Found at https://cs. uwaterloo. ca/\~{} browndg/negbin. pdf}, 6.

\bibitem[Charikar et~al., 2005]{charikar2005clustering}
Charikar, M., Guruswami, V., and Wirth, A. (2005).
\newblock Clustering with qualitative information.
\newblock {\em Journal of Computer and System Sciences}, 71(3):360--383.

\bibitem[Chu et~al., 2016]{chu2016distributed}
Chu, X., Ilyas, I.~F., and Koutris, P. (2016).
\newblock Distributed data deduplication.
\newblock {\em Proceedings of the VLDB Endowment}, 9(11):864--875.

\bibitem[Cochinwala et~al., 2001]{cochinwala2001efficient}
Cochinwala, M., Kurien, V., Lalk, G., and Shasha, D. (2001).
\newblock Efficient data reconciliation.
\newblock {\em Information Sciences}, 137(1-4):1--15.

\bibitem[Cohen, 1998]{cohen1998integration}
Cohen, W.~W. (1998).
\newblock Integration of heterogeneous databases without common domains using
  queries based on textual similarity.
\newblock In {\em ACM SIGMOD Record}, volume~27, pages 201--212. ACM.

\bibitem[Demaine et~al., 2006]{demaine2006correlation}
Demaine, E.~D., Emanuel, D., Fiat, A., and Immorlica, N. (2006).
\newblock Correlation clustering in general weighted graphs.
\newblock {\em Theoretical Computer Science}, 361(2-3):172--187.

\bibitem[Elmagarmid et~al., 2007]{elmagarmid2007duplicate}
Elmagarmid, A.~K., Ipeirotis, P.~G., and Verykios, V.~S. (2007).
\newblock Duplicate record detection: A survey.
\newblock {\em IEEE Transactions on knowledge and data engineering},
  19(1):1--16.

\bibitem[Hern{\'a}ndez and Stolfo, 1995]{hernandez1995merge}
Hern{\'a}ndez, M.~A. and Stolfo, S.~J. (1995).
\newblock The merge/purge problem for large databases.
\newblock In {\em ACM Sigmod Record}, volume~24, pages 127--138. ACM.

\bibitem[Ilyas et~al., 2015]{ilyas2015trends}
Ilyas, I.~F., Chu, X., et~al. (2015).
\newblock Trends in cleaning relational data: Consistency and deduplication.
\newblock {\em Foundations and Trends{\textregistered} in Databases},
  5(4):281--393.

\bibitem[Jaro, 1980]{jaro1980unimatch}
Jaro, M.~A. (1980).
\newblock {\em UNIMATCH, a Record Linkage System: Users Manual}.
\newblock Bureau of the Census.

\bibitem[Kulis et~al., 2009]{kulis2009semi}
Kulis, B., Basu, S., Dhillon, I., and Mooney, R. (2009).
\newblock Semi-supervised graph clustering: a kernel approach.
\newblock {\em Machine learning}, 74(1):1--22.

\bibitem[Levenshtein, 1966]{levenshtein1966binary}
Levenshtein, V.~I. (1966).
\newblock Binary codes capable of correcting deletions, insertions, and
  reversals.
\newblock In {\em Soviet physics doklady}, volume~10, pages 707--710.

\bibitem[Mitzenmacher and Upfal, 2005]{mitzenmacher2005probability}
Mitzenmacher, M. and Upfal, E. (2005).
\newblock {\em Probability and computing: Randomized algorithms and
  probabilistic analysis}.
\newblock Cambridge university press.

\bibitem[Monge et~al., 1996]{monge1996field}
Monge, A.~E., Elkan, C., et~al. (1996).
\newblock The field matching problem: Algorithms and applications.
\newblock In {\em KDD}, pages 267--270.

\bibitem[Shalev-Shwartz and Ben-David, 2014]{shalev2014understanding}
Shalev-Shwartz, S. and Ben-David, S. (2014).
\newblock {\em Understanding machine learning: From theory to algorithms}.
\newblock Cambridge university press.

\bibitem[Smith and Waterman, 1981]{waterman1981identification}
Smith, T. and Waterman, M. (1981).
\newblock Identification of common molecular subsequence.
\newblock {\em J Mol. Biol}, 147.

\bibitem[Vapnik and Chervonenkis, 2015]{vapnik2015uniform}
Vapnik, V.~N. and Chervonenkis, A.~Y. (2015).
\newblock On the uniform convergence of relative frequencies of events to their
  probabilities.
\newblock In {\em Measures of complexity}, pages 11--30. Springer.

\end{thebibliography}
\bibliographystyle{apalike}

\appendix
\section{Hardness of PCC in the presence of an oracle}
\begin{theorem}
Given that the Exponential Time Hypothesis (ETH) holds then any algorithm for the Promise Correlation Clustering problem  that runs in polynomial time makes $\Omega(|X|)$ same-cluster queries for all $M \ge 3$ and for $\alpha = 0$ and $\beta = \frac{1}{2}$. 
\end{theorem}

\noindent The exponential time hypothesis says that any solver for $3$-SAT runs in $2^{o(m)}$ time (where $m$ is the number of clauses in the $3$-SAT formula). We use a reduction from $3$-SAT to 3DM to X3C to show that the exact cover by 3-sets (X3C) problem also can't be solved in $2^{o(m)}$ time (if ETH holds). Then, using the reduction from the previous section implies that PCC also can't be solved in $2^{o(n)}$ time. Thus, any query based algorithm for PCC needs to make atleast $\Omega(n)$ queries where $n = |X|$ is the number of vertices in the graph. 

\begin{definition}[3-SAT].\\
Input: A boolean formulae $\phi$ in 3CNF with $n$ literals and $m$ clauses. Each clause has exactly three literals. \\
Output: YES if $\phi$ is satisfiable, NO otherwise. 
\end{definition}

\noindent\textbf{Exponential Time Hypothesis}\\
There does not exist an algorithm which decides 3-SAT  and runs in $2^{o(m)}$ time.

\begin{definition}[3DM].\\
Input: Sets $W, X$ and $Y$ and a set of matches $M \subseteq W \times X \times Y$ of size $m$.  \\
Output: YES if there exists $M' \subseteq M$ such that each element of $W, X, Y$ appears exactly once in $M'$. NO otherwise. 
\end{definition}

\noindent To prove that (X3C) is NP-Hard, the standard We will reduce 3-SAT to 3-dimensional matching problem. 3DM is already known to be NP-Hard. However, the standard reduction of 3-SAT to 3DM constructs a set with $|M| \in \Theta(m^2 n^2)$. Hence, using the standard reduction, the exponential time hypothesis would imply there does not exist an algorithm for 3DM which runs in $\Omega(m^\frac{1}{4})$. Our reduction is based on the standard reduction. However, we make some clever optimizations especially in the way we encode the clauses. This helps us improve the lower bound to $\Omega(m)$.

\begin{figure}[!ht]
	\centering
	\begin{tikzpicture}
	  \draw (0,0) node {$\bullet$} -- (1,0) node{$\bullet$} -- (0.5,0.5) node{$\bullet$} -- cycle;
	  \draw (1,0) node{$\bullet$} -- (2,0) node{$\bullet$}
  -- (2,-1) node{$\bullet$} -- cycle;
	  \draw (2,-1) node{$\bullet$} -- (2,-2) node{$\bullet$} -- (2.5,-1.5) node{$\bullet$} -- cycle;
	  \draw (2,-2) node{$\bullet$} -- (1,-3) node{$\bullet$} -- (2,-3) node{$\bullet$} -- cycle;
	  \draw (0,-3) node{$\bullet$} -- (1,-3) node{$\bullet$} -- (0.5,-3.5) node{$\bullet$} -- cycle;
	  \draw (0,-3) node{$\bullet$} -- (-1,-2) node{$\bullet$} -- (-1,-3) node{$\bullet$} -- cycle;
	  \draw (-1,-2) node{$\bullet$} -- (-1,-1) node{$\bullet$} -- (-1.5,-1.5) node{$\bullet$} -- cycle;
	  \draw (0,0) node{$\bullet$} -- (-1,-1) node{$\bullet$} -- (-1,0) node{$\bullet$} -- cycle;
	
	 \node  at (0.0,-0.3) {$b_4$};
	 \node  at (1.0,-0.3) {$a_1$};
	 \node  at (1.0,-0.6) {\scriptsize{(or $a_5$)}};
	 \node  at (2.0,0.3) {$c_1$};
	 \node  at (1.8,-1.2) {$b_1$};
	 \node  at (1.7,-1.9) {$a_2$};
	 \node  at (2.8,-1.5) {$c_1'$};
	 \node  at (2.0,-3.3) {$c_2$};
	 \node  at (1.0,-2.7) {$b_2$};
	 \node  at (0.0,-2.7) {$a_3$};
	 \node  at (0.5,-3.8) {$c_2'$};
	 \node  at (-1.0,-3.3) {$c_3$};
	 \node  at (-0.7,-2.0) {$b_3$};
	 \node  at (-0.7,-1.2) {$a_4$};
	 \node  at (-1.8,-1.6) {$c_3'$};
	 \node  at (-1.0,0.3) {$c_4$};
	 \node  at (0.5,0.8) {$c_4'$};
	 
	 \node  at (2.5,1.2) {$tf_1$};
	 \node  at (3.8,0.7) {$tf_1'$};
	 \node  at (4.3,0.0) {$tf_2$};
	 \node  at (4.3,-0.8) {$tf_2'$};
	 \node  at (4.3,-1.8) {$t_1$};
	 \node  at (3.4,-3.2) {$t_1'$};

	 \draw [line width=1mm] (4,0) -- (4,-1);
	 \draw [line width=1mm] (2.5,1) -- (3.5,0.5);
	 \draw [line width=1mm] (4,-2) -- (3,-3);
	 
  	\draw[dotted](4, 0) -- (2, 0);
  	\draw[dotted](4, -1) -- (2, 0);
  	\draw[dotted](4, 0) -- (2.5, -1.5);
  	\draw[dotted](4, -1) -- (2.5, -1.5);
  	\draw[dotted](4, 0) -- (5.1, 0.5);
  	\draw[dotted](4, 0) -- (5.1, 0.3);
  	\draw[dotted](4, 0) -- (5.1, -0.5);
  	\draw[dotted](4, 0) -- (5.1, -0.3);
  	\draw[dotted](4, -1) -- (5.1, 0.1);
  	\draw[dotted](4, -1) -- (5.1, -0.9);
  	\draw[dotted](4, -1) -- (5.1, -1.5);
  	\draw[dotted](4, -1) -- (5.1, -1.3);

  	\draw[dotted](2.5, 1) -- (2, 0);
  	\draw[dotted](3.5, 0.5) -- (2, 0);
  	\draw[dotted](2.5, 1) -- (2.5, -1.5);
  	\draw[dotted](3.5, 0.5) -- (2.5, -1.5);
  	\draw[dotted](2.5, 1) -- (3.3, 1.5);
  	\draw[dotted](2.5, 1) -- (3.3, 1.3);
  	\draw[dotted](2.5, 1) -- (3.5, 0.9);
  	\draw[dotted](2.5, 1) -- (3.5, 0.8);
  	\draw[dotted](3.5, 0.5) -- (4.2, 1.1);
  	\draw[dotted](3.5, 0.5) -- (4.2, 0.9);
  	\draw[dotted](3.5, 0.5) -- (3.5, 1.5);
  	\draw[dotted](3.5, 0.5) -- (4.5, 0.5);

  	\draw[dotted](4, -2) -- (2.5, -1.5);
  	\draw[dotted](3, -3) -- (2.5, -1.5);
  	\draw[dotted](4, -2) -- (5.1, -2);
  	\draw[dotted](4, -2) -- (5.1, -2.5);
  	\draw[dotted](3, -3) -- (4.2, -3);
  	\draw[dotted](3, -3) -- (4.2, -2.7);
  	\end{tikzpicture}
  	\caption{Part of graph $G$ constructed for the literal $x_1$. The figure is an illustration for when $x_1$ is part of four different clauses. The triangles (or hyper-edge) $(a_i, b_i, c_i)$ capture the case when $x_1$ is true and the other triangle $(b_i, c_i', a_{i+1})$ captures the case when $x_1$ is false. Assuming that a clause $C_j = \{x_1, x_2, x_3\}$, the hyper-edges containing $tf_i, tf_i'$ and $t_1, t_1'$ capture different settings. The hyper-edges containing $t_1, t_1'$ ensure that atleast one of the literals in the clause is true. The other two ensure that two variables can take either true or false values.}
\label{fig:3DMQueries}
\end{figure}

Our gadget is described in Fig. \ref{fig:3DMQueries}. For each literal $x_i$, let $m_i$ be the number of clauses in which the the literal is present. We construct a ``truth-setting" component containing $2m_i$ hyper-edges (or triangles). We add the following hyper-edges to $M$.
\begin{align*}
  &\{(a_k[i], b_k[i], c_k[i]): 1 \le k \le m_i\} \cup \{(a_{k+1}[i], b_k[i], c_k'[i]): 1 \le k \le m_i\}
\end{align*}
Note that one of $(a_k, b_k, c_k)$ or $(a_{k+1}, b_k, c_k')$ have to be selected in a matching $M'$. If the former is selected that corresponds to the variable $x_i$ being assigned true, the latter corresponds to false. This part is the same as the standard construction. 

For every clause $C_j = \{x_1, x_2, x_3\}$ we add three types of hyper-edges.  The first type ensures that atleast one of the literals is true. 
$$\{(c_k[i], t_1[j], t_1'[j]): x_i' \in C_j\} \cup \{(c_k'[i], t_1[j], t_1'[j]): x_i \in C_j\}$$ 
The other two types of hyper-edges (conected to the $tf_i$'s) say that two of the literals can be either true or false. Hence, we connect them to both $c_k$ and $c_k'$
\begin{align*}
  &\{(c_k[i], tf_1[j], tf_1'[j]): x_i' \text{ or }x_i\in C_j\} \cup \{(c_k[i], tf_2[j], tf_2'[j]): x_i \text{ or }x_i' \in C_j\}\\
  &\cup \{(c_k'[i], tf_1[j], tf_1'[j]): x_i' \text{ or }x_i\in C_j\} \cup \{(c_k'[i], tf_2[j], tf_2'[j]): x_i \text{ or }x_i' \in C_j\}
\end{align*}
Note that in the construction $k$ refers to the index of the clause $C_j$ in the truth-setting component corresponding to the literal $x_i$. Using the above construction, we get that
\begin{align*}
  & W = \{c_k[i], c_k'[i]\}\\
  & X = \{a_k[i]\} \cup \{t_1[j], tf_1[j], tf_2[j]\}\\
  & Y = \{b_k[i]\} \cup \{t_1'[j], tf_1'[j], tf_2'[j]\}
\end{align*} 
Hence, we see that $|W| = 2\sum_i m_i = 6m$. Now, $|X| = |Y| = \sum_i m_i + 3m = 6m$. And, we have that $|M| = 2\sum_i m_i + 15m = 21m$. Thus, we see that this construction is linear in the number of clauses. 

Now, if the 3-SAT formula $\phi$ is satisfiable then there exists a matching $M'$ for the 3DM problem. If a variable $x_i = T$ in the assignment then add $(c_k[i], a_k[i], b_k[i])$ to $M'$ else add $(c_k'[i], a_{k+1}[i], b_k[i])$. For every clause $C_j$, let $x_i$ (or $x_i'$) be the variable which is set to true in that clause. Add $(c_k'[i], t_1[j], t_1'[j])$  (or $(c_k[i], t_1[j], t_1'[j])$) to $M'$. For the remaining two clauses, add the hyper-edges containing $tf_1[j]$ and $tf_2[j]$ depending upon their assignments. Clearly, $M'$ is a matching. 

Now, the proof for the other direction is similar. If there exists a matching, then one of $(a_k, b_k, c_k)$ or $(a_{k+1}, b_k, c_k')$ have to be selected in a matching $M'$. This defines a truth assignment of the variables. Now, the construction of the clause hyper-edges ensures that every clause is satisfiable.

\begin{theorem}
If the exponential time hypothesis holds then there does not exist an algorithm which decides the three dimensional matching problem 3DM and runs in time $2^{o(m)}$.
\end{theorem}

\begin{corollary}
\label{cor:X3CLowerBound}
If the exponential time hypothesis holds then there does not exist an algorithm which decides exact cover by 3-sets problem (X3C) and runs in time $2^{o(m)}$.
\end{corollary}

Hence, from the discussion in this section, we know that X3C is not only NP-Hard but the running time is lower bounded by $\Omega(2^m)$. Now, using the same reduction of X3C to PCC as before, gives the same lower bound on the running time of PCC. Using this, we can now lower bound the number of queries required by PCC.

For the sake of contradiction, let us assume that there exists an algorithm which solves PCC in polynomial time by making $o(n)$ same-cluster queries ($n$ is the number of vertices). Then by simulating all possible answers for the oracle, we get a non-query algorithm which solves PCC in $2^{o(n)}$. However, combining Cor. \ref{cor:X3CLowerBound} with the reduction of X3C to PCC, we get that any algorithm that solves PCC takes $\Omega(2^n)$. Hence, no such query algorithm exists. 

\section{Sampling positive pairs}
\weightedPosApproxUniform*
\begin{proof}
Let $K = \{(x, y):d(x, y)\le \lambda\}$ and  $D$ be a distribution over $K$ defined by $D(x, y) := \frac{|S_x|}{\sum_{x'} |S_{x'}|} . \frac{1}{|S_x|} = \frac{U^2(x,y)}{U^2(K)}$. Let $K^+ = \{(x, y) : d(x, y) \le \lambda$ and $C^*(x, y) = 1\}$. Let $T$ be the distribution induced by $\mc P_{11}$. It's easy to see that for $(x, y) \not\in K^+$, $T(x, y) = 0$. For $(x, y) \in K^+$, let $D(x, y) = p$ and $D(K^+) = q$. Then, $T(x, y) = p + (1-q)p + \ldots = \frac{p}{q} = \frac{D(x, y)}{D(K^+)} = \frac{U^2(x, y)}{U^2(K^+)}$. Using Defn. \ref{defn:informativeMetric}, we know that 
\begin{align*}
  &1-\alpha \enspace \le \underset{(x, y) \sim U^2}{\mb P} [d(x, y) \le \lambda \enspace|\enspace C^*(x, y) = 1] = \frac{\underset{(x, y) \sim U^2}{\mb P} [ d(x, y) \le \lambda, C^*(x, y) = 1]}{\underset{(x, y) \sim U^2}{\mb P} [C^*(x, y) = 1]} = \frac{U^2(K^+)}{U^2(X^{[2]+})}\numberthis\label{eqn:wtIneq}
\end{align*}
Now, we will use the above inequality to prove our result. 
\begin{flalign*}
  &\underset{(x, y) \sim T}{\mb P}\enspace \big[ h(x, y) = 0 ] = \underset{(x, y)\in K^+}{\sum} T(x, y) \mb 1_{h(x, y) = 0} =\underset{(x, y) \in K^+}{\sum} \frac{U^2(x, y)}{U^2(K^+)} \mb 1_{h(x, y) = 0} \le\enspace \frac{1}{1-\alpha}\underset{(x, y) \in K^+}{\sum} \frac{U^2(x, y)}{U^2(X^{[2]+})}  \mb 1_{h = 0}\\
  & \le (1+2\alpha)\underset{(x, y) \in X^2_+}{\sum} P^+(x, y) \mb 1_{h(x, y) = 0} =\enspace (1+2\alpha)\underset{(x, y) \sim P^+}{\mb P}\enspace \big[ h(x, y) = 0 ]&
\end{flalign*}
Now, for the other direction, we have that 
\begin{align*}
  &\underset{(x, y) \sim P^+}{\mb P}\enspace \big[ h(x, y) = 0 ] = \underset{(x, y):X^{[2]+}}{\sum} P^+(x, y) \mb 1_{h(x, y) = 0} = \underset{ (x, y) \in K^+}{\sum} \frac{U^2(x, y)}{U^2(X^{[2]+})} \mb 1_{h(x, y) = 0} + \underset{(x, y)\in X^2_+ \setminus K^+}{\sum} \frac{U^2(x, y)}{U^2(X^{[2]+})} \mb 1_{h=0}\\
  & \le \underset{ (x, y) \in K^+}{\sum} \frac{U^2(x, y)}{U^2(K^+)} \mb 1_{h(x, y) = 0}  + \underset{(x, y) \in X^{[2]+} \setminus K^+}{\sum} \frac{U^2(x, y)}{U^2(X^{[2]+})} \mb 1_{h = 0}\\
  & \le \underset{(x, y) \sim T}{\mb P}\enspace \big[ h(x, y) = 0 ] + \underset{(x, y)\in X^{[2]+} \setminus K^+}{\sum} \frac{U^2(x, y)}{U^2(X^{[2]+})}\enspace\le\enspace  \underset{(x, y) \sim T}{\mb P}\enspace \big[ h(x, y) = 0 ] + \alpha
\end{align*}
Hence, we have shown that both the directions hold and this completes the proof of the lemma. Note that this shows that our sampling procedure approximates the distribution $P^+$. It is easy to see that pre-computing $S_x$ for all $x$ takes $|X|^2$ time. Once the pre-computation is done, the sampling can be done in constant time.
\end{proof}

\section{Sample and query complexity of RCC}
\sampleComplexity*
\begin{proof}
Let $T_0$ be the distribution induced by $\mc P_0$ and $T_1$ be the distribution induced by $\mc P_{11}$. Denote by $E(h) = \underset{(x, y) \sim P^+}{\mb P}\enspace \big[ h(x, y) = 0 ]$ and by $G(h) = \underset{(x, y) \sim P^-}{\mb P}\enspace \big[ h(x, y) = 1 ]$. 

Using Thm. \ref{thm:uniformConvergence}, we know that if $m_+ > a\frac{\vcdim({\mc F}) + \log(\frac{1}{\delta})}{\epsilon^2}$ then with probability atleast $1-\delta$, we have that for all $h$
\begin{align*}
  &|\hat E(h) - \underset{(x,y) \sim T_1}{\mb P}[h(x, y) = 0]| \le \epsilon \implies \hat E(h) \le \epsilon + \underset{(x,y) \sim T_1}{\mb P}[h(x, y) = 0] \le \epsilon + 2\alpha + E(h) \enspace\text{ and}\\
  &E(h) - 2\alpha -\epsilon \le \hat E(h) \numberthis \label{eqn:eHat}
\end{align*}
Note that we obtain upper and lower bounds for $\underset{(x,y) \sim T_1}{\mb P}[h(x, y) = 0]$ using Lemma \ref{lemma:weightedPosApproxUniform}. Similarly, if $m_- > a\frac{\vcdim({\mc F}) + \log(\frac{1}{\delta})}{\epsilon^2}$, then with probability atleast $1-\delta$, we have that for all $h$,
\begin{align*}
  &|\hat G(h) - \underset{(x,y) \sim T_0}{\mb P}[h(x, y) = 1]| \le \epsilon \implies \hat G(h) \le \epsilon + G(h) \enspace\text{ and } \enspace G(h) -\epsilon \le \hat G(h) \numberthis\label{eqn:gHat}
\end{align*}

\noindent Combining Eqns. \ref{eqn:eHat} and \ref{eqn:gHat}, we get that with probability atleast $1-2\delta$, we have that for all $C \in {\mc F}$
\begin{align*}
  &\hat L(C) \le \mu [\epsilon + E(h) + 2\alpha] + (1-\mu)(\epsilon + G(h)) \le L(h) + \epsilon + 2\alpha\\
  &\text{And} \enspace \hat L(C) \ge \mu(E(h) -\epsilon - \alpha) + (1-\mu)(G(h) - \epsilon) \ge L(h) - \epsilon - \alpha
\end{align*}
Now, let $\hat C$ be the output of $\mc A$ and let $\hat C^*$ be $\argmin_{C \in {\mc F}} L(C)$. Then, we have that with probability atleast $1-2\delta$
\begin{align*}
  L(\hat C) &\le \hat L(\hat C) + \alpha + \epsilon \le \hat L(\hat C^*) + \alpha + \epsilon \le L(\hat C^*) + 2\epsilon + 3\alpha 
\end{align*}

Choosing $\epsilon = \frac{\epsilon}{2}$ and $\delta = \frac{\delta}{2}$ throughout gives the result of the theorem.
\end{proof}

\queryComplexity*
\begin{proof}
Let $q_+$ denote the number queries to sample the set $S_+$. We know that $\mb E[q_+] \le \frac{1}{\beta}$. Given that the expectation is bounded as above, using Thm. \ref{thm:geometricRV}, we get that $q_+ \le \frac{(1+\nu)m_+}{\beta(1-\epsilon)}$ with probability atleast $1-\exp(\frac{-\nu^2m_+}{4})$. Similarly, we get that with probability atleast $1-\exp(\frac{-\nu^2m_-}{4})$, $q_- \le \frac{(1+\nu)m_-}{(1-\gamma)}$.
\end{proof}

\subsection{VC-dimension of common classes}
\VCDim*
\begin{proof}
Let $n$ be as defined in the statement of the theorem. Let $M^2 \subseteq \mc X^2$ be a set of size $> n$. Define $M := \{x: (x, y) \in M^2 \text{ or } (y, x) \in M^2\}$. We know that $|M| > \sqrt n$. The number of clusterings (partitions) on $n$ elements is given by the $n^{th}$ bell number. Thus, for $s \le B_{\sqrt n}$ there exists a clustering $C' \not\in \mc F$ of the set $\mc X$. Hence, $l_{\mc F}$ can't shatter any set of size $> n$.
\end{proof}

\begin{lem}
\label{lemma:treesOnX}
Let $\mc X$ be a finite set, $S \subseteq \mc X$ be a set of $n$ points and $T$ be any hierarchical clustering tree of $\mc X$. There exists a set $\mc C = \{C_1, \ldots, C_s\}$ where each $C_i$ is a clustering of $S$ with the following properties
\begin{itemize} 
	\item $|\mc C| \ge \frac{n!}{\lfloor n/2 \rfloor! \enspace 2^{\lfloor n/2 \rfloor}}$
	\item $T$ contains atmost one clustering from $\mc C$. 
\end{itemize}
\end{lem}
\begin{proof}
Consider clusterings $C_i$ of $S$ of the following type. Each cluster in $C_i$ contains exactly two points (except possibly one cluster which contains one point if $n$ is odd). One such clustering along with a tree $T$is shown in Fig. \ref{fig:treeStructure}. Let $\mc C$ be the set of all such clusterings $C_i$. The number of such clusterings $|\mc C|$ is 
$$ \begin{dcases}
		\frac{n!}{2^{\frac{n-1}{2}}\frac{n-1}{2}!} \hspace{0.5in} n \text{ is odd}\\
        \frac{n!}{2^{\frac{n}{2}}\frac{n}{2}!} \hspace{0.78in} n \text{ is even}\\
	\end{dcases}= \frac{n!}{2^{\lfloor \frac{n}{2} \rfloor} (\lfloor \frac{n}{2} \rfloor)!}$$
	
For the sake of contradiction, assume that $T$ is a hierarchical clustering tree $T$ of $\mc X$ which contains $C_i$ and $C_j$. Since $C_i \neq C_j$, there exists points $s_1, s_2$ and $s_3$ such that the following happens. (i) $s_1, s_2$ are in the same cluster in $C_i$. $s_2, s_3$ as well as $s_1, s_3$ are in different clusters in $C_i$. (ii) $s_1, s_3$ are in the same cluster in $C_j$. $s_2, s_3$ as well as $s_1, s_2$ are in different clusters in $C_j$.   

Now, $T$ contains $C_i$. Hence, there exists a node $v$ such that $s_1, s_2 \in C(v)$ but $s_3 \not\in C(v)$. $T$ also contains $C_j$. Hence, there exists a node $u$ such that $s_1, s_3 \in C(u)$ and $s_2 \not\in C(u)$. Both $u$ and $v$ contain the point $s_1$. Hence, either $u$ is a descendant of $v$ or the other way around. Observe that $s_2 \in C(v)$ but $s_2 \not\in C(u)$. Hence, $v$ is not a descendant of $u$. Similarly, $s_3 \in C(u)$ and $s_3 \not\in C(v)$ so $u$ is not a descendant of $v$. This leads to a contradiction. Hence, no such tree $T$ can exist.  
\begin{figure}
	\centering
	\begin{tikzpicture}
		\node[circle,draw] at (0,0) {};
		
		\node[circle,draw] at (1, 0){};
		\node[circle,draw,fill=black] at (1.5, 1){};
		\draw (1.5, 1) --(1.95, 0.2);
		\draw (1.5, 1) --(1.05, 0.2);
		\node[circle,draw] at (2,0) {};
		
		\node[circle,draw,fill=black] at (2.5, 2){};
		\draw (2.5, 2) --(1.5, 1);
		\draw (2.5, 2) --(3.5, 1);

		\node[circle,draw] at (3, 0){};
		\node[circle,draw,fill=black] at (3.5, 1){};
		\draw (3.5, 1) --(3.95, 0.2);
		\draw (3.5, 1) --(3.05, 0.2);
		\node[circle,draw] at (4,0) {};
		
		\node[circle,draw,fill=black] at (4.5, 3){};
		\draw (4.5, 3) --(2.5, 2);
		\draw (4.5, 3) --(6.5, 2);

		\node[circle,draw,fill=black] at (4.5, 5){};
		\draw (4.5, 5) --(4.5, 3);
		\draw (4.5, 5) --(0, 0);

		\node[circle,draw] at (5, 0){};
		\node[circle,draw,fill=black] at (5.5, 1){};
		\draw (5.5, 1) --(5.95, 0.2);
		\draw (5.5, 1) --(5.05, 0.2);
		\node[circle,draw] at (6,0) {};
		
		\node[circle,draw,fill=black] at (6.5, 2){};
		\draw (6.5, 2) --(5.5, 1);
		\draw (6.5, 2) --(7.5, 1);

		\node[circle,draw] at (7, 0){};
		\node[circle,draw,fill=black] at (7.5, 1){};
		\draw (7.5, 1) --(7.95, 0.2);
		\draw (7.5, 1) --(7.05, 0.2);
		\node[circle,draw] at (8, 0){};
		
		\draw[dotted] (0, 1.5) -- node[above] {$C_i$} (8, 1.5);
	\end{tikzpicture}
\caption{A hierarchical clustering tree of $n=9$ points. This tree contains the clustering $C_i$ described in the proof of Lemma \ref{lemma:treesOnX}.}
\label{fig:treeStructure}
\end{figure}
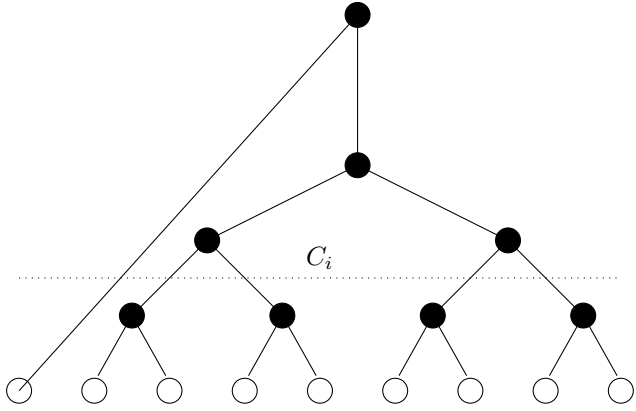
\end{proof}

\VCDimT*
\begin{proof}
Let $n$ be as defined in the statement of the theorem. Let $M^2 \subseteq \mc X^2$ be a set of size $> n^2$. Define $M := \{x: (x, y) \in M^2 \text{ or } (y, x) \in M^2\}$. We know that $|M| > n$. Using lemma \ref{lemma:treesOnX}, there exists a set of clusterings $\mc C = \{C_1, \ldots, C_{s'}\}$ of size $s' > \frac{n!}{\lfloor n/2 \rfloor! \enspace 2^{\lfloor n/2 \rfloor}} \ge s$ such that each $T_i \in \mc F$ contains atmost one $C_j \in \mc C$. Thus, there exists a clustering $C_j$ which is not captured by any $T_i \in \mc F$. Hence, $l_{\mc F}$ can't shatter any set of size $> n^2$.
\end{proof}

\section{Technical lemmas and theorems}

\begin{theorem}[Multiplicative Chernoff bound \cite{mitzenmacher2005probability}]
\label{thm:chernoff}
Let $X_1, \ldots, X_n$ be i.i.d random variables in $\{0, 1\}$ such that $\mu = E[X_i]$. Let $X = \frac{\sum X_i}{n}$. Then for any $0 < \epsilon < 1$
$$P\big[ \enspace X > (1+\epsilon) \mu\enspace\big] \enspace\le\enspace \exp\bigg(\frac{-\epsilon^2\mu n}{3}\bigg)$$
\end{theorem}

\begin{theorem}[Multiplicative Chernoff bound \cite{mitzenmacher2005probability}]
\label{thm:chernoff2}
Let $X_1, \ldots, X_n$ be i.i.d random variables in $\{0, 1\}$ such that $\mu = E[X_i]$. Let $X = \frac{\sum X_i}{n}$. Then for any $0 < \epsilon < 1$
$$P\big[ \enspace X < (1-\epsilon) \mu\enspace\big] \enspace\le\enspace \exp\bigg(\frac{-\epsilon^2\mu n}{2}\bigg)$$
\end{theorem}

\begin{theorem}[Vapnik and Chervonenkis \cite{vapnik2015uniform}]
\label{thm:vceapprox}
Let $X$ be a domain set and $D$ a probability distribution over $X$. Let $H$ be a class of subsets of $X$ of finite VC-dimension $d$. Let $\epsilon, \delta \in (0,1)$. Let $S \subseteq X$ be picked i.i.d according to $D$ of size $m$. If $m > \frac{c}{\epsilon^2}(d\log \frac{d}{\epsilon}+\log\frac{1}{\delta})$, then  with probability $1-\delta$ over the choice of $S$, we have that $\forall h \in H$
$$\bigg|\frac{|h\cap S|}{|S|} - P(h)\bigg| < \epsilon$$
\end{theorem}

\begin{theorem}[Fundamental theorem of learning \cite{blumer1989learnability}] 
\label{thm:uniformConvergence}
Here, we state the theorem as in the book \cite{shalev2014understanding}. Let $H$ be a class of functions $h:\mc X \rightarrow \{0, 1\}$ of finite VC-Dimension, that is $\vcdim(H) = d < \infty$. Let $D$ be a probability distribution over $X$ and $h^*$ be some unknown target function. Given $\epsilon, \delta \in (0, 1)$. Let $err_D$ be the $\{0, 1\}$-loss function $err: H \rightarrow [0, 1]$. That is $err_D(h) = \underset{x \in D}{\mb P}[h(x) \neq h^*(x)]$. Sample a set $S = \{(x_1, y_1), \ldots, (x_m, y_m)\}$ according to the distribution $D$. Define $err_S(h) = \sum_{i=1}^{m} \frac{\mb 1_{[h(x_i) \neq h^*(x_i)]}}{m}$. If $m \ge a\frac{d + \log (1/\delta)}{\epsilon^2}$, then with probability atleast $1-\delta$ over the choice of $S$, we have that for all $h \in H$
$$|err_D(h) - err_S(h)| \le \epsilon$$
where $a$ is an absolute global constant. 
\end{theorem}

\begin{theorem}[Concentration inequality for sum of geometric random variables \cite{brown2011wasted}]
\label{thm:geometricRV}
Let $X = X_1 + \ldots + X_n$ be $n$ geometrically distributed random variables such that $\mb E[X_i] = \mu$. Then 
$$\mb P[X > (1+\nu)n\mu] \le \exp\bigg(\frac{-\nu^2\mu n}{2(1+\nu)}\bigg)$$
\end{theorem} 
\end{document}